%% file: ms.tex
\documentclass{article}

\usepackage{microtype}
\usepackage{graphicx}
\usepackage{subfigure}
\usepackage{booktabs} % for professional tables
\usepackage{tcolorbox}
\usepackage{amsmath, amsthm}
\usepackage{amssymb}
\usepackage{tikz}
\usetikzlibrary{backgrounds}

\usepackage{hyperref}

\newtheorem{theorem}{Theorem}
\newtheorem{definition}{Definition}[section]
\newtheorem{lemma}[theorem]{Lemma}
\newtheorem{corollary}[theorem]{Corollary}
\usepackage{xspace}
\newcommand{\NP}{\textit{NP}\xspace}
\newcommand{\coNP}{\textit{coNP}\xspace}
\newcommand{\PP}{\textit{P}\xspace}

\newcommand{\supplementary}{Appendix}

\usepackage[inline]{enumitem}

\usepackage[accepted]{icml2020}

\icmltitlerunning{It's Not What Machines Can Learn, It's What We Cannot Teach}

\newif\ifhardlanguage
\hardlanguagetrue  
\newif\iflongcqc
\longcqcfalse

\begin{document}

\twocolumn[
\icmltitle{It's Not What Machines Can Learn, It's What We Cannot Teach}



\icmlsetsymbol{equal}{*}

\begin{icmlauthorlist}
\icmlauthor{Gal Yehuda}{tech}
\icmlauthor{Moshe Gabel}{to}
\icmlauthor{Assaf Schuster}{tech}
\end{icmlauthorlist}

\icmlaffiliation{to}{Department of Computer Science, University of Toronto, Canada} 
\icmlaffiliation{tech}{Department of Computer Science, Technion -- Israel Institute of Technology, Israel}

\icmlcorrespondingauthor{Gal Yehuda}{ygal@cs.technion.ac.il}
\icmlcorrespondingauthor{Moshe Gabel}{mgabel@cs.toronto.edu}

\icmlkeywords{Machine Learning, ICML}

\vskip 0.3in
]



\printAffiliationsAndNotice{}  % leave blank if no need to mention equal contribution

\begin{abstract}
\input{abstract}
\end{abstract}

\section{Introduction}
\input{introduction}

\input{CQC}

\input{inherent_bias}

\section{Related Work}
\input{related_work}

\section{Discussion} \label{discussion}
\input{discussion}

\section*{Acknowledgements}
We wish to thank Ernest Davis, Hadar Sivan, Daniyal Liaqat, and Shunit Agmon for valuable discussions and feedback.
We also wish to thank the ICML anonymous reviewers for their comments.

\bibliography{ms}
\bibliographystyle{icml2020}

\section*{\supplementary{}}

\input{nonsample}

\input{cqc_details}
\end{document}

%% file: abstract.tex
Can deep neural networks learn to solve any task, and in particular problems of high complexity?
This question attracts a lot of interest, with recent works tackling computationally hard tasks such as the traveling salesman problem and satisfiability.
In this work we offer a different perspective on this question.
Given the common assumption that $\NP \neq \coNP$ we prove that any polynomial-time sample generator for an \NP-hard problem samples, in fact, from an easier sub-problem.
We empirically explore a case study, Conjunctive Query Containment, and show how common data generation techniques generate biased datasets that lead practitioners to over-estimate model accuracy.
Our results suggest that machine learning approaches that require training on a dense uniform sampling from the target distribution cannot be used to solve computationally hard problems, the reason being the difficulty of generating sufficiently large and unbiased training sets.

%% file: introduction.tex
Applying deep learning methods to solve or approximately solve\footnotemark computationally hard problems has gained popularity in recent years.
Examples include attempts to solve the satisfiability problem \cite{selsam2018learning,DBLP:conf/aaai/CameronCHL20}, the traveling salesman problem (TSP) \cite{prates2019learning, milan2017data} and symbolic integration \cite{lample2019deep}.
There has also been recent interest in developing dedicated architectures for learning how to perform algorithmic tasks from solved instances, such as the Neural Turing Machine \cite{DBLP:journals/corr/GravesWD14}, the Differentiable Neural Computer \cite {graves2016hybrid}, and the Neural GPU \cite{kaiser2015neural}. 

\footnotetext{Throughout the paper, by ``solve'' we mean solve or solve approximately, e.g., by allowing some level of error.}

The \emph{expressive power} of deep neural networks, which represents the breadth of functions deep models are able to compute, has been an active area of research since the rise of deep learning \cite{siegelmann1991turing, raghu2017expressive,lu2017expressive,Xu2020What}.  
We know that recurrent neural networks and many modern architectures are Turing complete \cite{prez2019turing} when allowed unbounded precision, meaning they are in theory capable of performing any computation that a Turing machine can do. 
This raises the intriguing possibility of discovering efficient approximate solvers by using machine learning to train a model on solved instances of a given problem. 
However, even if a model is expressive enough in theory, we must also be able to train it to arrive at the correct solution.

\begin{figure}
    \centering
    \centerline{\includegraphics{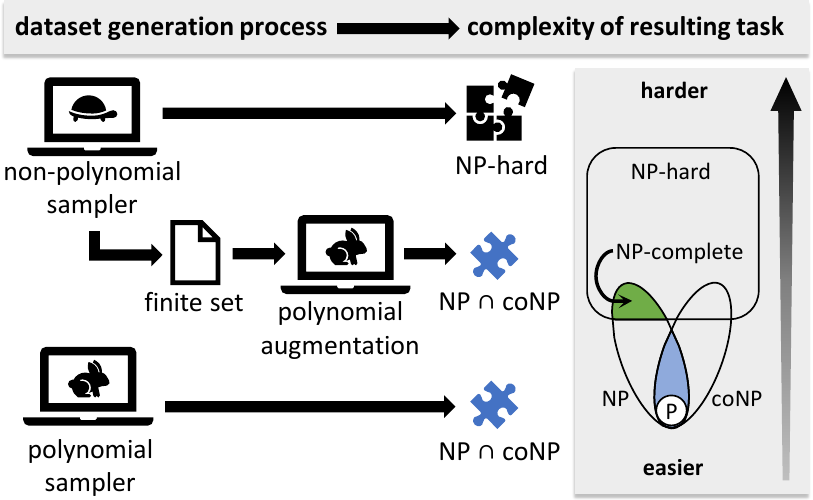}}
    \caption{Assuming $\NP \neq \coNP$, any polynomial-time data generator for a \NP-hard classification task will output data from an easier $\NP \cap \coNP$ task, even when starting with a sample generated by a non-polynomial deterministic process and augmenting it.
    }
    \label{fig:summary}
\end{figure}

The difficulty emerges in acquiring a suitable dataset. 
Large, diverse and densely sampled datasets are essential for the learning ability of deep learning models \cite{10.5555/3203489}.
Existing datasets for computational tasks tend to be application specific; 
such datasets may also be biased towards a subset of the problem space, which may be easy and unrepresentative.
For example, training a model to answer the 3-SAT problem using a dataset where all examples follow a simple pattern may yield high accuracy for similar data without capturing the full essence and difficulty of the problem in the resulting model. 
A trivial example of such a pattern is when all the positive instances are shorter than the negative instances. 
A more subtle case is when all samples are instances of an easy problem that is hard to identify at first glance yet for which an efficient solution is known, such as 2-colorability.
In both cases, solving the problem on these datasets does not mean solving the broader problem.
Since the performance of such models is measured empirically, a biased, possibly easy dataset may lead us to falsely believe the models are solving the general problem.

For abstract computational tasks such as 3-SAT and TSP, a popular alternative to using existing datasets is generating solved instances~\cite{selsam2018learning,prates2019learning}.
Such dataset generators can generate as many samples as we wish, which is particularly appealing when training models that require large training sets. 
Moreover, performance evaluation can be  more precise since we can generate as many samples as we need to reduce the generalization gap. 

Dataset generators, however, are not without issues.
Labeling datasets for \NP-hard classification tasks requires deterministic solvers whose runtime grows exponentially or worse with the problem size~\cite{kovacs2013first}, which is impractical for the large training sets needed by popular ML approaches.
Instead, practitioners turn to alternative approaches that run in polynomial (often linear) time.
One common approach is starting with a random example and carefully applying transformations so that the label is known by construction~\cite{lample2019deep}.
Another approach is data augmentation: start with a seed set of deterministically-labeled samples and apply class-preserving rewrites~\cite{selsam2018learning}.
It is not uncommon for test sets to be generated using the same procedure.

\paragraph{Our Contributions}
We show that polynomial-time dataset generators cannot be used to train models in solving \NP-hard problems.
If a classification task is an \NP-hard decision problem, any \emph{efficient} (polynomial time) procedure generates biased, unrepresentative data sets of solved instances, unless $\NP = \coNP$. 
In other words, when starting with an \NP-hard  problem, the data sampling procedure leaves us with an easier problem that we train the model on.
Figure~\ref{fig:summary} illustrates this result.
\ifhardlanguage
Finally, we show an example of the worst case scenario: an \NP-hard language for which any polynomial-time dataset generator creates a trivial classification task.
\fi

Specifically, under the commonly accepted assumption that $\NP \neq \coNP$ we prove the following:
\begin{enumerate}[noitemsep, nosep]

\item No polynomial-time data generation procedure can ever sample from the full problem space.

\item The classification task that a polynomial-time data generator can sample from is an $\NP \,\cap\, \coNP$ decision problem, strictly easier than the original problem.

\ifhardlanguage
\iffalse
\item There is an \NP-hard language for which samples generated by any polynomial-time procedure can, with high probability, be classified using superficial features.
\else
\item There is a language that is \NP-hard to decide, yet any polynomial-time procedure that generates samples from it creates samples that can with high probability be classified using a superficial feature.
\fi
\fi
\end{enumerate}

As a case study, we consider the  \NP-complete problem of Conjunctive Query Containment, or CQC~\cite{chandra1977optimal,Chirkova2018}.
We use a data augmentation approach to quickly generate large training sets of solved CQC instances, and train a neural network model to solve it.
We demonstrate how training on the generated dataset is not enough for solving the original CQC task, and that using the same procedure to generate the test set can lead us to overestimate model performance.

In summary, we show a kind of ``Catch-22'' for \NP-hard problems: even if we had the right model architecture and training algorithm, we cannot feasibly obtain the data required in order to train them.
Though a trained model may appear to solve the task on an efficiently generated dataset, it does not mean the trained model has learned to solve the original task.

%% file: CQC.tex
\section{Case Study: Learning an \NP-hard Problem} \label{CQC}

In this section we demonstrate how common and seemingly reasonable data generation approaches can cause us to overestimate model performance. 
We describe a representative case study: modeling, training, and evaluating a solver for the Conjunctive Query Containment (CQC) problem.
CQC is a central problem in the theory of databases \cite{Chirkova2018}, motivated by both practical and theoretical interests, with applications in query minimization and optimization \cite{jarke1984query}, verifying data integrity \cite{fernandez1999verifying}, cache management \cite{draper2001nimble} and querying incomplete databases \cite{imielinski1988incomplete}.

\subsection{Problem Definition} 
The problem of query containment is to decide, given two database queries $p, q$, if for every database $D$ the results of $p$ on $D$ are contained in the results of $q$ on $D$.
For clarity, we focus on a simpler yet \NP-complete version of this problem, with up to 2 relations and no projections.

A \emph{database} $D = \{R_1,...,R_k\}$ is a collection of tables, where each table $R_i$ is collection of rows (tuples) of length 3.
A \emph{conjunctive query} $q$ over the database $D$ is a first order predicate of the form 
\begin{align*}
\label{eq:predicate}
    \exists x_1,...,x_n : R_{i_1}(\ell_1, \ell_2, \ell_3)\wedge \cdots \wedge 
                            R_{i_s}(\ell_{3s-2}, \ell_{3s-1}, \ell_{3s})
\end{align*}
Where $x_1,...,x_n$ are variables, and $\ell_j$ is either a variable (some $x_u$) or a constant $c_w$.   
We assume that all variables and constants take value from a finite set $\Sigma$.
Given a conjunctive query $q$, we denote by $vars(q)$ the set of all variables in $q$.
\iflongcqc
An example of a query with 4 conjunctions is:
\begin{align*}
q    : \; & R_1(x_4, x_3, x_1)  \, \wedge \\
                              & R_1(x_1, x_2, x_1)  \, \wedge \\
                              & R_2 (x_2, x_3, x_3) \, \wedge \\
                              & R_2 (x_5, x_1, x_2)
\end{align*}
\else
For example, the following is a query with 3 conjunctions:
\begin{align*}
q : R_1(x_4, x_3, x_1)  \wedge  R_1(x_1, x_2, x_1)  \wedge R_2 (x_5, x_1, x_2)
\end{align*}
\fi
A tuple $(c_1,...,c_n)$ \emph{satisfies} the query $q$ for database $D$ if when assigning $c_i$ to $x_i$ the predicate is true.
The \emph{evaluation} of a query $q$ on a database $D$, denoted by $q(D)$, is the collection of all tuples which satisfy $q$.

Conjunctive Query Containment (CQC) is the set of all pairs $(p, q)$ of conjunctive queries such that $p(D) \subseteq q(D)$ for every database $D$; we denote such pairs by $p \subseteq q$.
Deciding whether a query pair $(p,q)$ is in CQC is \NP-complete~\cite{chandra1977optimal}.

\subsection{An RNN Model for CQC} \label{model_architecture}

Exact containment is \NP-complete, so instead we aim to give an approximation using supervised learning: we will train a model to discriminate between query pairs.
Given two queries $q$ and $p$ as a sequence of tokens, it will output 1 if $q \subseteq p$ or 0 if $q \not \subseteq p$.

\iflongcqc
\input{cqc_token_encoding}
\fi

\paragraph{Input Encoding}
Given a pair of conjunctive queries $(p,q)$ and a binary label, we tokenize each query and represent it as a fixed length sequence of one-hot vectors with 42 dimensions (the number of tokens in our dictionary). 
The sequence length is $95$, since this is the longest possible query with our parameters.
We pad shorter queries with zero vectors.
\iflongcqc
Table \ref{tab:model:representation} shows the mapping between the query tokens and their representation.
\else
The full table of token encodings is available in the \supplementary{}.
\fi

\paragraph{Model Architecture}
Since we aim to map sequences (query pairs) to scalars, we choose to use Recurrent Neural Networks with Long Short-Term Memory (LSTM) units\footnotemark, which excel at such tasks and have been shown to be computationally expressive~\cite{prez2019turing, weiss-etal-2018-practical}.
\footnotetext{
We emphasize that our main results in Section~\ref{bias} do not depend on any particular modeling choice, and apply equally to all approaches that require dense sampling.
Nevertheless, we have also explored alternatives including Transformers and learned embeddings, with no meaningful difference in empirical performance or generalization. We discuss hybrid architectures in Section \ref{discussion}.}

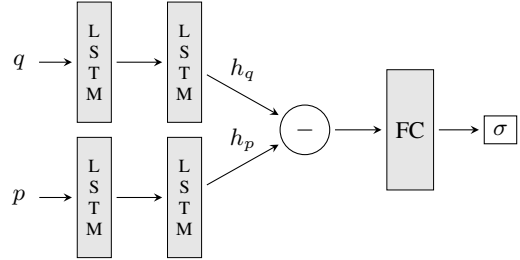
\begin{figure}
    \centering
    \input{figs/tikz_architecture}
    \caption{Model architecture}
    \label{fig:model:architecture}
\end{figure}

Figure~\ref{fig:model:architecture} shows the network architecture for the model.
We encode each query into a $w$-dimensional vector using LSTM layers with ReLU activations: two layers for $p$ and two layers for $q$.
The length of each layer is $n$, and the internal dimension (width) of the LSTM units is $w$.
The final LSTM state vectors $h_p$ and $h_q$ are then subtracted from each other, resulting in the $w$-dimensional vector $v = h_q - h_p$.
Finally, The vector $v$ is fed to a fully connected layer that reduces $v$ to a single scalar (i.e., a dot product), followed by a sigmoid activation function $\operatorname{\sigma}(x)$ to normalize the output to the range $[0, 1]$.
When $p \subseteq q$ the label will be 1, or 0 otherwise.

\subsection{Data Generation}
\label{sec:data-generation}

Simply generating random query pairs and labeling them using a determinstic CQC solver is not feasible, given the number of pairs we need and the large size of each query, and may also result in an unbalanced training set.

One common approach is to generate one input from the pair and work forwards or backwards to the other input by applying a sequence of rewrites that guarantee the pair's class~\cite{lample2019deep}.
However this approach risks introducing superficial features and biasing the data towards unrealistic examples~\cite{davis2019use}.

Instead, we aim to sample query pairs directly.
We address class imbalance by sampling $(p,q)$ from a special distribution $\mu$ such that $\Pr\left[ p \subseteq q \right] \approx 0.5$, yet both positive and negative instances have the same structure (size, number of variables, etc.).
We first generate a small ``seed'' set of query pairs by sampling from $\mu$ and labeling them using a deterministic theorem prover. 
We then use \emph{data augmentation} to generate large training sets -- a common approach for this problem~\cite{selsam2018learning}.
% SRC (page 8, Section 8) of selsam2018) is absolutely augmentation. Then they train on it
% even SR (section 4) is in some way forward/backward because the positive of the pair is generated from the negative.

\paragraph{Generating Balanced Dataset} \label{mu}
When drawing samples from parametrized distribution, 
many \NP-complete languages such as 3-SAT and TSP exhibits a \emph{phase transition} phenomenon:
the likelihood of a random sample drawn from a special parametric distribution to be in the language is determined by where the distribution's parameter $\alpha$ is in relation to constant $c$  \cite{gent1994sat,zhang2004phase, prates2019learning}.
We exploit a similar phenomena in CQC to draw balanced samples.

\iflongcqc
\input{cqc_sampling_mu}

\else
We define a parametric family of query pairs $\mu(m_1,m_2)$ such that sampling $(p,q)$ from $\mu(m_1,m_2)$ with $m_1 \geq m_2$ guarantees the following properties.
%\begin{itemize}[nosep, noitemsep]
    %\item 
    First,
    $p$ has $m_1$ conjunctions and $q$ has $m_2$ conjunctions.
    %\item 
    Second, the probability that $p \subseteq q$ is approximately 0.5.
    %\item 
    Finally, the process for generating positive and negative examples is the same.
%\end{itemize}
The definition and details of $\mu(m_1,m_2)$ are available in the \supplementary{}.
\fi

We generate instances of $(p,q)$, both positive and negative, where the number of conjunctions in $p$ is 1--10, and the number of conjunctions in $q$ is 1--8. 
We first choose $m_1 \sim U(1,8)$ and $m_2 \sim U(1,\min\{m_1,8\})$,
and then sample $(p,q)\sim \mu(m_1,m_2)$.
For each conjunction we choose a relation at random from
$R=\{R_1,R_2\}$, with 3 variables or constants sampled uniformly with repetition from the set $\{x_0, \dots, x_{32},0,1\}$.
Using $R=\{R_1,R_2\}$ is sufficient to make the problem \NP-complete. %~\cite{chandra1977optimal}.}
%When $m_1, m_2$ are such that $X_2$ needs more than 33 variables, we rename them to variables no used in $X_1$.
%
We use the Vampire theorem prover \cite{kovacs2013first} to obtain the label for each sample.

\paragraph{Data Augmentation}
Though the time complexity of sampling from $\mu$ is linear, generating large training sets this way is infeasible since the deterministic theorem prover runs in exponential time in the worst case.

Instead, we augment every labeled sample in the seed set to create $99$ additional samples with the same label.
Starting with the original sample $(p,q)$, we apply a sequence of up to 3 randomly selected class-preserving rewrites, yielding a new pair $(p',q')$ with the same label.
We repeat the process 98 more times, each time starting from the last $(p',q')$.
Since the original seed set was balanced, this results in a dataset $\times 100$ larger with roughly half positive and half negative instances. 
Data augmentation runs in linear time.

\iflongcqc
\input{cqc_data_augmentation}
\else
An example of a class preserving rewrite is variable merging: if $p \subseteq q$, then merging two variables in $p$ to a single variable will preserve the containment.
The full list of all class-preserving rewrites for $p \subseteq q$ and $p \not\subseteq q$ is available in the \supplementary{}.
\fi

\subsection{Experimental Results}

\begin{figure}
  \centering
  \centerline{\includegraphics[width=\columnwidth]{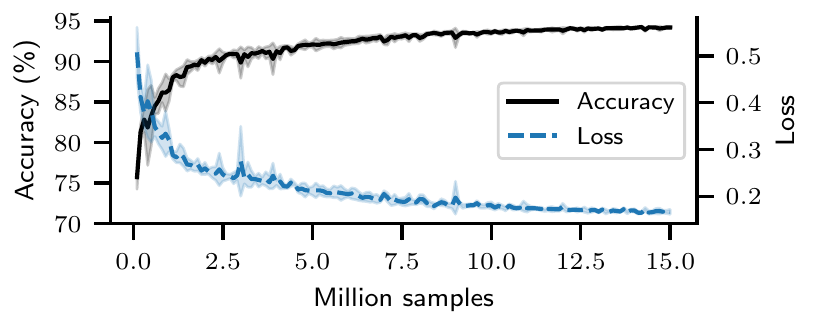}}
  \caption{Average model accuracy and cross entropy loss throughout the learning process for the unseen \textsc{aug} test set comprised of 500K queries. Lines show average of 5 models, bands show standard deviations. }
  \label{plot:model1_curve}
\end{figure}

We trained 5 models using the Adam optimizer \cite{kingma2014adam}, with binary cross entropy loss.
We set the dimensionality of the LSTM output space to $w=256$, and learning rate was set to 0.00105 by tuning on a separate validation set. 
Adam's hyperparameter $\beta_1$  was set to $0.9$ and $\beta_2$ was set to $0.999$. %(the defaults).
We train each model for 150 steps: in each step we generate 100K query pairs and train with mini-batch size of 500. We used a 3.3GHz Intel i9-7900X machine with two Nvidia GeForce GTX 1080 Ti GPUs.
%We generated data on a Intel(R) Core(TM) i7-7500U CPU @ 2.70GHz machine with $32$gb RAM.

Figure~\ref{plot:model1_curve} shows average performance during training, measured on the \textsc{aug} test set: a balanced test set of 500K instances generated 
by applying the data augmentation procedure to a new seed set (Section~\ref{sec:data-generation}).
The average final accuracy after 15 million samples is $94.2\%$ (SD 0.6\%).

\paragraph{Generalization}

While the model appears to perform very well on the unseen test set, we were suspicious.
Is it really possible that such a straightforward model results in such high accuracy?

\begin{figure}
  \centering
  \centerline{\includegraphics[width=\columnwidth]{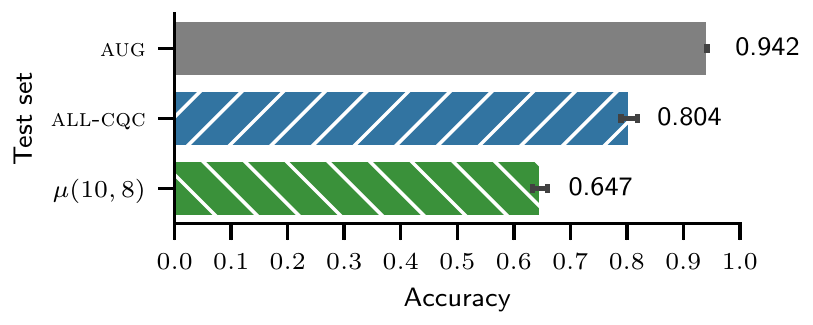}}
  \caption{Average final accuracy for different test sets. Error bars show standard deviation. The high performance on the test set generated by data augmentation method does not translate to high performance on the other test sets. }
  \label{plot:model1_train-test}
\end{figure}

To test generalization, we generated two additional test sets.
The first one, denoted \textsc{all-cqc}, is the set of all 537,477,120 conjunctive query pairs with $2$ conjunctions in $p$ and $2$ conjunctions in $q$, labeled by a deterministic solver.
The second dataset, denoted $\mu(10, 8)$, contains 250K queries sampled from $\mu$, where $p$ had $10$ conjunctions and $q$ had $8$ conjunctions, again labeled by the solver.

Figure~\ref{plot:model1_train-test} shows the average accuracy of the trained model on the two new test sets, as well as the original test set based on mutating pairs of conjunctive queries.
The high accuracy obtained on the original test set is not preserved when testing it on the entire space.
Additionally, it is worth noting that \textsc{all-cqc} is unbalanced: % due to the small size of queries we could sample: 
classifying everything as 0 would result in accuracy above 90\%.
Performance on the balanced %representative
$\mu(10,8)$ dataset is even lower, even though its class balance matches that of the training set.

\subsection{Discussion}
What went wrong?
Clearly the model has learned something: it performs very well on an unseen test set created by our data generator.
This suggests the issue is not improper learning schedule or a poor choice of model. 
Instead, the model did not learn how to solve CQC but rather how to exploit a property of the generation method.
Moreover, by generating the test set using the same procedure, we overestimate performance on the full problem space. 
Had we not tested on \textsc{all-cqc} and $\mu(10,8)$, we might have remained convinced that the model has learned to solve CQC.

In the next section, we show that the issue indeed lies with data generation, and that it would be difficult to overcome for any modeling approach that requires large training sets.
Any polynomial-time data generation method for an \NP-hard problem results in an easier sub-problem.

%% file: cqc_token_encoding.tex
\iflongcqc
\begin{table}
\else
\begin{table}[h]
\fi
    \caption{Token representation. Each token with index $j$ is mapped to a vector with $1$ in position $j$ and all other elements are zero. The dictionary size and the length of the vectors is $d=42$.
    }
    \label{tab:model:representation}
    \vskip 0.1in
    \begin{center}
    \begin{small}
    \begin{tabular}{lcl}
        \toprule
        Type & Tokens & Index range \\
        \midrule
        Variables & \texttt{x0} ~ \dots ~ \texttt{x32} & 6--11, 14-40 \\
        Relations & \texttt{Q} ~ \texttt{R0} ~ \texttt{R1}  & 12, 5, 4 \\
        Operators & \texttt{$\wedge$} ~ \texttt{:} & 1, 13 \\
        Parentheses & \texttt{(} ~ \texttt{)} & 2, 3 \\
        Constants & \texttt{0} ~ \texttt{1} &  41, 42 \\
        \bottomrule
    \end{tabular}
    \end{small}
    \end{center}
    \vskip -0.1in
\end{table} 

%% file: figs/tikz_architecture.tex
\begin{tikzpicture}[every node/.style={outer sep=2, font=\footnotesize}]
\usetikzlibrary{decorations.pathreplacing}

\tikzstyle{unit} = [draw, circle, minimum size=14, inner sep=2]
\tikzstyle{arrow} = [-stealth]
\tikzstyle{layer} = [draw, very thin, fill=black!10, minimum height=1.6cm, minimum width=0.4cm, align=center]
\tikzstyle{lstm} = [layer,font=\scriptsize]

\node [circle,draw, outer sep=2] (minus) at (-0.4,3.7) {$-$};
\node [layer] (fc) at (1,3.7) {FC};
\node [draw] (sigma) at (2.2,3.7) {$\sigma$ };

\node [lstm] (qlstm2) at (-2,4.6) {L\\S\\T\\M};
\node [lstm] (qlstm1) at (-3.2,4.6) {L\\S\\T\\M};
\node [lstm] (plstm2) at (-2,2.8) {L\\S\\T\\M};
\node [lstm] (plstm1) at (-3.2,2.8) {L\\S\\T\\M};
\node [shift={(-1,0)}] (q) at (qlstm1) {$q$};
\node [shift={(-1,0)}] (p) at (plstm1) {$p$};

\draw [arrow] (qlstm1) -- (qlstm2);
\draw [arrow] (plstm1) -- (plstm2);
\draw [arrow] (q) -- (qlstm1);
\draw [arrow] (p) -- (plstm1);

\draw [arrow] (qlstm2) edge node [midway, above] {$h_q$} (minus);
\draw [arrow] (plstm2) edge node [midway, above] {$h_p$} (minus);

\draw [arrow] (minus) -- (fc);
\draw [arrow] (fc) -- (sigma);

\end{tikzpicture}

%% file: cqc_sampling_mu.tex
Intuitively, for a conjunctive query $p$ with a fixed number of conjunctions, the fewer variables is uses, the more ``constrained'' it is. 
For example, let $p(x_1) = R_1(x_1, x_2, x_3)$ and $q(x_1) = R_1(x_1, x_1, x_2)$.
While every tuple in $R_1$ will satisfy $p$, only tuples whose first and second element are the same will satisfy $q$.

Given a fixed set of relations $R$, we define the distribution $G(X, m)$ over conjunctive queries with $m$ conjunctions, where $X$ is a set of variables as follows:
first, choose $m$ relations from $R$ uniformly and with repetitions;
then, conjunction variables for each conjunction uniformly and with repetitions from $X$. 
The \emph{constraintness} of $G(X,m)$ is defined as $\alpha = \frac{m}{n}$.

Let $p \sim G(X_1, m_1)$ and $q \sim G(X_2, m_2)$ be a query pair, and let $\alpha_1$ and $\alpha_2$ be the respective constraintness.
We observe that the probability of $p \subseteq q$ depends on the ratio of $\alpha_2$ and $\alpha_1$. 
When $\frac{\alpha_2}{\alpha_1} \gg c$ for a constant $c$, with high probability $p \subseteq q$, when $\frac{\alpha_2}{\alpha_1} \ll c$ with high probability $p \not\subseteq q$, and when $\frac{\alpha_2}{\alpha_1} \approx c$, the probability of $p \subseteq q$ is approximately $0.5$.
We empirically determined that for $m_1 \geq m_2$,  $c \approx \frac{2}{15}$. 

Finally, we define the distribution $\mu(m_1,m_2)$ over pairs of conjunctive queries $(p,q)$ as sampling $p \sim G(X_1, m_1)$ and $q \sim G(X_2, m_2)$ with $X_1$ and $X_2$ such that $\frac{\alpha_2}{\alpha_1} \approx c$. 
Since positive and negative samples are generated with the same structure and the same constraintness, syntactic features alone are unlikely to help classification.

%% file: cqc_data_augmentation.tex
Given a query $q$, we define the following rewrites:
\begin{itemize}[nosep, noitemsep]
    \item \texttt{MergeVar($q$)}: Pick two variables $x,y \in vars(q)$, replace every occurrence of $y$ by $x$.
    \item \texttt{SplitVar($q$)}: Pick a new variable $w \not\in vars(q)$, and a variable $x \in vars(q)$. Each occurrence of $x$ is unchanged with probability $0.5$ or replaced with $w$.
    \item \texttt{AddConj($q$)}: Pick a conjunction $R(\ell_1, \ell_2,\ell_3)$ and add it to $q$.
    \item \texttt{DelConj($q$)}: Pick a conjunction in $p$ and remove it. 
    \item \texttt{Shuffle($q$)}: Shuffle the order of conjunctions in $p$.
\end{itemize}

For $(p,q)$ where $p \subseteq q$, we use the following set of class-preserving rewrites:
$(\texttt{MergeVar}(p), q)$, 
$(p, \texttt{SplitVar}(q))$, 
$(\texttt{AddConj}(p), q)$, 
$(p, \texttt{DelConj}(q))$,
$(\texttt{Shuffle}(p), q)$, and $(p, \texttt{Shuffle}(q))$. 
For $(p,q)$ where $p \not\subseteq q$, we use the following class-preserving rewrites: 
$(p, \texttt{MergeVar}(q))$, 
$(\texttt{SplitVar}(p), q)$, 
$(p, \texttt{AddConj}(q))$,
$(\texttt{Shuffle}(p), q)$, and $(p, \texttt{Shuffle}(q))$.

%% file: inherent_bias.tex
\section{Inherent Bias in Efficient Samplers} \label{bias}

Supervised learning requires obtaining a training set: instances of the problem with known labels.
When the training set is biased or the label leaks via superficial features, such as sample length or range of attributes, the resulting model may be of no value.

In this section we show that any polynomial-time data generation method for an \NP-hard decision problem is not just inherently biased, it is also biased in a way that precludes training a model to solve the original problem.

We study the data generation problem for binary classification tasks under the assumption that $\NP \neq \coNP$, and show that for an \NP-hard decision problem, no efficient method can produce every possible labeled instance.
Moreover, the dataset generated by any efficient data generation method for \NP-hard problems will provably generate data from an easier sub-problem of the original classification task.
Hence, learning to (approximately) solve the sampled sub-problem does not guarantee learning to solve the original problem, since there are two different classification problems.
\ifhardlanguage
In addition, we show an example of a problem for which any efficient data generation procedure only generates data that is trivially solvable with arbitrarily high probability, whereas the original problem cannot be solved in polynomial time.
\fi

\subsection{Complete Efficient Samplers Do Not Exist}
We first discuss desirable properties for data generation methods for classification tasks, and show that under the assumption that $\NP \neq \coNP$, it is impossible to obtain both efficient and representative generators for \NP-hard problems. 

A \textit{language} $L$ is a set of strings.  
Every language $L$ induces a binary classification task: the positive class contains all the strings in $L$, and the negative class contains all the strings in $L$'s complement $L^C$.
For example, $L = \text{CQC}$ is the set of all strings $w$ such that $w = (p,q)$ for two conjunctive queries $p,q$ and $p \subseteq q$. 
The classification task induced by $L$ is to decide, given a string $w$, whether $w$ is in $L$.

A \emph{sampler} $S_L$ for the classification problem induced by the language $L$ is a randomized algorithm (an algorithm that can flip coins during its execution) which generates labeled samples from both the negative and the positive classes.
Our goal is to obtain a dataset as representative as possible in the following sense: achieving high accuracy on the sampled instances should indicate high accuracy on the entire space of instances.

Hence, a reasonable property of a representative sampler is \emph{completeness}: the ability to generate every instance with a non-zero probability (otherwise a classifier trained using the sampler may have low accuracy on the unsampled parts of the problem space).
In addition, since modern machine learning methods such as deep neural networks require large datasets, the sampler is used to generate millions of labeled instances. Thus, a sampler should be \emph{efficient}, which we define as polynomial run time complexity\footnote{In reality, this is hardly sufficient for real use, but as we see even this permissive requirement is too demanding for samplers.}. 

The first question we address is the existence of such efficient and complete samplers for \NP-hard problems.
Alas, such samplers do not exist: under the plausible assumption that $\NP \neq \coNP$, we will now show that it is impossible to obtain efficient and complete data samplers for \NP-hard languages.
Without loss of generality and for technical convenience, we separate our discussion between a \emph{positive} and a \emph{negative} sampler.

\begin{definition}[Positive Sampler]
\label{def:positive-sampler}
A positive sampler $S^{+}_L$ for the language $L$, is a randomized algorithm which on input $n$ (represented in unary), outputs a string $w$ such that $|w| = n$ and $w \in L$, or outputs $\emptyset$ if no such string exists. %\todo{perhaps explain what we mean by randomized algorithm (draws a sequence of random bits?}
\end{definition}

%Positive samplers can only sample instances from the positive class, whereas deep learning models require samples from the negative class as well, which motivates the following definition.

\begin{definition}[Negative Sampler] A negative sampler $S^{-}_L$ for a language $L$ is a positive sampler for $L^C$: on input $n$  $S^{-}_L$ outputs a string $w$ such that $|w| = n$ and $w \not\in L$, or outputs $\emptyset$ if no such string exists.
\end{definition}

Optionally, a sampler (negative or positive) can also output the string of random bits drawn by the sampler when generating $w$.
For a language $L$ with both a positive and a negative samplers, we define a \emph{sampler} for $L$.

\begin{definition}[Sampler]
\label{def:sampler}
A sampler $S_L$ for a language $L$ is a randomized algorithm such that on input $n$ (represented in unary) it samples a word $w$ using either $S_L^{+}$ or $S_L^{-}$ and returns $w$ and the corresponding label 1 or 0.
\end{definition}

Note that this definition matches any data generation algorithm for $L$, regardless of method, since we do not limit how $S_L$ chooses which sampler to use. It can even run both $S_L^+$ and $S_L^-$, and only then choose which word to output.
Moreover, our definition of sampler implies generating both a sample and its correct label.
Being able to sample from a space does not necessarily imply knowing the label of the result. 
For example, we can easily generate a random Boolean 3-CNF formula, without knowing whether it is satisfiable. 
However, for supervised learning, we would still have to label it using a deterministic solver.
Thus, sampling from the space of 3-CNF formulas is not the same as sampling from the space of 3-SAT formulas, where the label is known. 
The former does not match our definition for a sampler (Definition~\ref{def:sampler}, while the latter is a positive sampler (Definition~\ref{def:positive-sampler}).

We denote by $S_L(n)$ the set of strings of length $n$ that can be generated by the sampler $S_L$. 
A sampler is \emph{complete} if it can generate every example: for every sufficiently large $n$, for every $w \in L$ of length $n$ it holds that $w \in S_L(n)$.
A sampler $S_L$ is called \emph{efficient} if it runs in polynomial time. 
For clarity, if $w \in S_L(|w|)$, in other words if $w$ can be generated by $S_L$, we denote it by $w \in S_L$.

The notion of complete efficient sampler is related to the notion of Nondeterministic Test Instance Construction Method (NTICM), as defined by \citet{sanchis1990complexity}.
The NTICM for a language $L$ is a nondeterministic Turing machine $M$ such that on input $1^n$ outputs a string from $L$, and that for every string in $L$ there is a computational path of $M$ which outputs it.
As proven in \cite{sanchis1990complexity}, NTICMs for \coNP-complete languages do not exist unless $\NP = \coNP$. 

We now show that the existence of efficient complete sampler implies the existence of NTICM.
It follows that efficient complete negative samplers for \NP-complete languages do not exist, hence efficient complete samplers do not exist.
For clarity, we first prove this result for \NP-complete problems. 
In Section~\ref{sec:sampler-bias} we extend it to all \NP-hard languages.

\begin{theorem}\label{thrm:non_existence} 
If $L$ is \NP-complete, then there is no efficient complete sampler for it, unless \NP = \coNP.
\end{theorem}

\begin{proof}
Assume by contradiction that there exists an efficient complete sampler $S_L$ for an \NP-complete language $L$.
Denote by $S^{-}_L$ the negative sampler used by $S_L$.
Define the following nondeterministic Turing machine $M$:
$M$ runs $S^{-}_L$ on $1^n$, and each time $S^{-}_L$ flips a coin, $M$ decides nondeterministicly to which branch of $S^{-}_L$ to proceed. 
$M$ is a NTICM for $L^C$, which is \coNP-complete since we assumed $L$ is \NP-complete, in contradiction to Proposition 2.1 in \cite{sanchis1990complexity}.
\end{proof}

Theorem \ref{thrm:non_existence} shows that it is impossible to obtain an efficient complete sampler for both the negative and the positive classes of an \NP-complete language $L$. 
We note that even the existence of efficient complete \emph{positive} samplers for all languages in \NP is an open problem:
%Though some known \NP-complete languages have an efficient PCS, there is no generic method to obtain such sampler for a given language in \NP, and ad-hoc techniques are needed.  
the existence of a language $L \in \NP$ with no efficient complete positive sampler would imply that $\PP \neq \NP$~\cite{sanchis1990complexity}.
%On the other hand, if all languages in \NP have an efficient PCS, then there are no sparse languages in $\NP \backslash \PP$ \cite{sanchis1990efficient}. 
%In addition, it is known that all languages in \NP have an efficient PCS if and only if all languages in \PP have an efficient PCS. 
%It means that even for a computationally tractable tasks, it is unclear if we can generate all the positive instances in an efficient manner. 

We next show that efficient data samplers for \NP-hard languages are biased towards an easier subset. 

\subsection{Incomplete Efficient Samplers are Biased}
\label{sec:sampler-bias}

We now show that not only an efficient sampler cannot generate the entire space of labeled instances for an \NP-hard problem, but also the instances it does generate are \emph{easier} to decide than the original problem.

\begin{definition}
The classification task induced by $S_L$, denoted by $C(S_L)$, is the task of classifying instances generated by $S_L$: given an instance $w$ generated by $S_L$ (without its label), determine if $w \in L$. 
\end{definition}

$C(S_L)$ may be easier than the original decision problem.
For example, let $L$ be the MAX-CLIQUE problem: given a graph $G$ and a number $k$, does $G$ has a clique of size $k$? 
Consider a sampler $S_L$ that generates instances of $(G,k)$ with matching labels, but can only generate planar graphs. 
In this case $C(S_L)$ would be the problem of deciding if $G$ has a clique of size $k$, where $G$ is a planar graph.
While the original classification task of deciding $L$ is \NP-complete, the task $C(S_L)$ is in \PP \cite{chiba1985arboricity}.

It turns out that if $L$ is \NP-hard, $C(S_L)$ is always easier for any polynomial-time $S_L$, assuming $\NP \neq \coNP$.

\begin{lemma} \label{lemma:npconp}
If $S_L$ is an efficient sampler for a language $L$, then the classification task $C(S_L)$ is in $\NP \cap \coNP$.
\end{lemma}

\begin{proof}

Recall that $L \in \NP$ if for every word $w\in L$ there exists a string $z$ with length polynomial in $|w|$ (a \emph{certificate}) such that a deterministic Turing machine (\emph{verifier}) that given $w$ and $z$ can verify in polynomial time that $w \in L$. 
Similarly, $L \in \coNP$ if there exist polynomial verification for every $w \notin L$.
Note it is enough to prove that the certificate and verifier exist, even if we do not know what they are.

Given a string $w$ generated by $S_L$ with label 1, let $z$ be the sequence of random bits used by $S_L$ to generate $w$. 
%Given a string $w$ generated by $S_L$ with label 1, a verification string $z$ that $w \in L$ is the sequence of random bits used by $S_L$ in order to generate $w$. 
We can now build a deterministic Turing machine $M$ that, given $w$ and $z$, verifies $w \in L$.
At each step, $M$ operates as $S_L$ would; whenever $S_L$ needs to draw a random bit, $M$ will use the next bit from $z$.
Since $S_L$ runs in polynomial time, it must use at most polynomial number of random bits. 
Once $S_L$ ends, $M$ verifies that its output is $w$ and the label returned by $S_L$ is 1.
Thus, for every $w \in S_L$ there exists a polynomial verification $z$ for $w \in L$.
Note this verifier only applies to $w$ generated by $S_L$, not to every $w \in L$.

Similarly, given a negative string $w$ generated by $S_L$, we can use the verifier $M$ and the sequence of random bits used by $S_L$ to verify in polynomial time that $w \not \in L$.

We thus conclude that $C(S_L) \in \NP \cap \coNP$, completing the proof.
\end{proof}

Lemma \ref{lemma:npconp} bounds the complexity of the classification problem over \emph{any} efficient sampler. 
In particular, under the assumption that $\coNP \neq \NP$, even if the original problem is \NP-hard, after sampling the classification task \emph{cannot} be \NP-hard:
%Unless $\NP = \coNP$, for every $L \in \NP$-hard,
there is no polynomial time reduction between solving $C(S_L)$ and solving $L$. %, hence solving $C(S_L)$ does not mean solving $L$. 

It immediately follows efficient samplers for \NP-hard languages sample from a strictly easier sub-problem $C(S_L)$.
The proof follows from Lemma~\ref{lemma:npconp} when $L$ is \NP-hard.

\begin{corollary} \label{thrm:bias_in_every_sampler}
If $L$ is an \NP-hard problem and assuming $\coNP \neq \NP$, then for \emph{any} efficient sampler $S_L$ for $L$ the classification task over $S_L$ is \emph{not} \NP-hard: $C(S_L) \notin \NP{\text{-hard}}$.
\end{corollary}

Corollary~\ref{thrm:bias_in_every_sampler} shows that the sampled sub-problem is easier. 
It also implies that even a machine learning model learns to correctly classify instances from $C(S_L)$, that model does not necessarily solve $L$, meaning that test sets generated by $S_L$ cannot be used to evaluate performance on $L$.

We can also use Lemma~\ref{lemma:npconp} to show an equivalent to Theorem~\ref{thrm:non_existence} for \NP-hard languages.
\begin{corollary} \label{thrm:np-hard-is-incomplete}
If $L$ is \NP-hard, then there is no efficient complete sampler for it, unless \NP = \coNP.
\end{corollary}
\begin{proof}
Assume by contradiction that $S_L$ is an efficient complete sampler for $L$.
Since $S$ is efficient, by Lemma~\ref{lemma:npconp} $C(S_L) \in \NP \cap \coNP$.
Since we assume $\NP \neq \coNP$, we have $\NP \cap \coNP \neq \NP$-hard, and therefore $C(S_L) \neq L$, which contradicts our assumption that $S_L$ is complete.
\end{proof}

\ifhardlanguage
We next show an extreme example of a language ${L_0}$ with severe bias. Any efficient sampler generates trivial examples, yet ${L_0}$ is a difficult problem.

\subsection{An \NP-hard Language with Trivially Decidable Instances}
\label{sec:tricky_langs}

Lemma~\ref{lemma:npconp} gives an upper bound on the difficulty of $C(S_L)$.
But what of the other direction? 
Given that $L$ is \NP-hard, how easy can $C(S_L)$ be, and can we meaningfully train a model to classify it?
In general, this depends on the language $L$ and the specific efficient sampler $S_L$.

However, we now give an example of a worst-case scenario: an \NP-hard language $L_0$ where if $S_{L_0}$ is an efficient sampler, then $C(S_{L_0})$ can be classified easily and with very high accuracy.
More precisely, we will show that any $w$ generated by any efficient sampler $S_L$ can be classified in constant time using a superficial feature.

This is somewhat surprising.
No matter how we implement an efficient sampler for ${L_0}$, the resulting training set will be useless to us: any such model trained on it simply learn to look at the superficial feature.
Note that our construction guarantees that the fraction of inputs that can be classified based on this superficial feature can be made arbitrarily small, so even a model that can perfectly classify the sub-problem will have arbitrarily small accuracy on the original problem.
Though deciding ${L_0}$ may not seem immediately practical, we conjecture that many \NP-hard languages may exhibit similar, though less extreme, properties: samplers that generate superficial features, or $C(S_{L_0})$ that is always be in \PP.

To prove this result,  we first prove the following Lemma. 
\begin{lemma} 
\label{lemma:non_sample}
There exists an \NP-hard language ${L_1}$ and a function $\delta(n) \rightarrow 0$ as $n \rightarrow \infty$, such that for any sufficiently long $w$ generated by any randomized polynomial process, 
$\Pr[ w \in L_1 ] \leq \delta(n)$ .
\end{lemma}

A full proof of Lemma~\ref{lemma:non_sample} is included the \supplementary{}.
Here we describe a sketch of the proof.

Let $M_1, M_2, \dots$ be an enumeration for all Turing machines.
We construct a randomized algorithm $P$ that runs in super-polynomial time: given size $n$, it chooses a Turing machine between $M_1\dots M_{g(n)}$ (where $g(n)$ grows slowly towards infinity), runs it, and returns its output $w$.
We then apply a result by~\citet{itsykson_et_al:LIPIcs:2016:6808} to show there is a process $P^*$ that is slower than $P$, but can with high probability $1-\epsilon(n)$ generate words that $P$ cannot.
Since the output of $P$ includes any polynomial process up to $g(n)$, we show that the probability for $P$ to generate $w$ is below $g(n)\epsilon(n) \rightarrow 0$ for $g(n)$ that grows sufficiently slowly.

We now use Lemma~\ref{lemma:non_sample} to to prove the following Theorem. 

\begin{definition}
Let $L$ be an \NP-hard language.
The polynomial sampler $S_L$ is \emph{trivial} if there exists $m$ such that for any word $w$ generated by $S_L$ where $|w|\geq m$, with high probability $w \in L$ if and only if the first bit of $w$ is $1$.
\end{definition}

\begin{theorem} \label{thrm:non_sample}
There exists an \NP-hard language $L_0$ for which every polynomial sampler $S_{L_0}$ for ${L_0}$ is trivial.
\end{theorem}

\begin{proof}
Let ${L_1}$ be the language from Lemma \ref{lemma:non_sample}. 
Define the language ${L_0}$ using the string concatenation operator $\circ$: 
\[ {L_0} = \{ 1 \circ u \mid u \not\in {L_1} \} \cup \{ 0 \circ v \mid v \in {L_1} \} \text{ .} \] 

Let $S_{L_0}$ be a polynomial sampler for $L_0$, and let $w$ of length $n$ be a word generated by $S_{L_0}$. 
Denote by $b$ the first bit of $w$, and by $x$ the last $n-1$ bits of $w$. 
Since $S_{L_0}$ runs in polynomial time, it follows from Lemma~\ref{lemma:non_sample} that with probability greater than $1- \delta(n-1)$, $x \not\in L_1$.

We now show that $S_{L_0}$ is \emph{trivial}: positive samples generated by the sampler $S_{L_0}$ start with $b=1$ with high probability, and that negative examples start with $b=0$ with high probability.
As $n$ grows, $\Pr[x \in L_1]$ shrinks. 
Thus positive examples generated by the sampler will be, with high probability, of the form $\{ 1 \circ u \mid u \not\in {L_1} \}$.
Conversely, negative examples generated by $S_{L_0}$
are from the form $ \{ 1 \circ u \mid u \in {L_1} \} \cup \{ 0 \circ v \mid v \not\in {L_1} \} $. As $n$ grows the probability that $w$ is of the form $\{ 1 \circ u \mid u \in {L_1} \}$ shrinks, thus with high probability $w \in \{ 0 \circ v \mid v \not\in {L_1} \}$.

The \NP-hardness of $L_0$ follows from the \NP-hardness of $L_1$. 
A reduction $R$ from $L_1$ to $L_0$ simply concatenates $0$ to a word $w$, $R(w) = 0 \circ w$. 
Then $w \in L_1$ if and only if $R(w) \in L_0$, which implies that $L_0$ is \NP-hard.
\end{proof}

It follows from Theorem~\ref{thrm:non_sample} that for every polynomial sampler $S_{L_0}$ for ${L_0}$, there is a simple algorithm that obtains arbitrarily high accuracy on the instances generated by $S_{L_0}$: return the first bit of the input.

%% file: related_work.tex
Using neural networks to solve computationally hard problems has been studied for many years, with early works attempting approximation of combinatorial optimization problems~\cite{anderson1988neural,Budinich1997NeuralNF,smith1999neural}.
Recent efforts on using deep neural networks to solve \NP-complete problems include 3-SAT~\cite{selsam2018learning}, graph problems~\cite{khalil2017learning, prates2019learning}, symbolic mathematics~\cite{lample2019deep}, and learning to solve routing problems \cite{kool2018attention}. 
An alternative research direction is hybrid architectures that incorporate deterministic solvers.
Solving \NP-complete problems with a differential solver layer was studied in \cite{pmlr-v97-wang19e, ferber2019mipaal}.  
\citet{selsam2019guiding} propose integrating deep learning models with deterministic solver in order to improve heuristics used by deterministic SAT solvers.

Data generation in these works was done either by deterministic solvers, which are impractical for large training sets, or by data augmentation heuristics.  
For example, in recent work \citet{lample2019deep} generate instances of symbolic integration problems by applying transformation to random samples, and by discovering new samples from existing ones (data augmentation).
As noted by~\citet{davis2019use}, these specific techniques are biased: the generated instances are not diverse and do not represent the difficulty of the problem. They might also leak information via the relative size of function pairs.

Though we focus on classic computational tasks, 
when real life decisions made by machine learning models trained on biased datasets, such models can perpetuate the bias in future decisions~\cite{yapo2018ethical}. 
The deleterious effects of dataset bias have been further documented when machine learning is used for healthcare~\cite{oakden2019hidden}, recidivism prediction~\cite{dressel2018accuracy}, predicting criminal behaviour~\cite{yapo2018ethical} and job performance~\cite{cawley2010over}. 
A survey on bias in machine learning can be found in~\cite{mehrabi2019survey}.

The study of generating solved instance for computationally hard problems was initiated by~\citet{sanchis1990complexity}, who studied the ability to generate optimization problems which are difficult for deterministic solvers.
This line of research focuses on generating particularly difficult (slow to compute) instances for deterministic solvers~\cite{selman1996generating, horie1997hard, cook1997finding, xu2005simple, haanpaa2006hard}, for example to benchmark solvers~\cite{DBLP:journals/corr/abs-1903-03592}.
In contrast, when generating data to train machine learning models we generally prefer unbiased, representative samples.

%% file: discussion.tex
Recent years have seen many attempts to use machine learning (ML), and in particular, deep neural networks (DNNs), to solve intractable computational tasks, if only approximately.
In parallel, there has been much discussion on what DNNs can do, and what they can learn.
We show that it is not enough to worry about the representation power of the network and the properties of the loss surface, but also the procedure used to generate the data we need to train it.

We prove that, under the common assumption that $\NP \neq \coNP$, any efficient sampling technique for an \NP-hard problem is hopelessly biased:
the probability of sampling from certain parts of the problem space is zero -- no efficient sampler is complete. 
Worse, the resulting sub-problem that we do sample from is strictly easier than the original \NP-hard problem.
Thus, common approaches to increasing training set size such as data augmentation result in a training set that does not reflect the full problem. 
Any ML model trained on such datasets does not learn to solve nor approximate the full NP-hard problem -- only the easier sub-problem.
\ifhardlanguage
Moreover, the sub-problem may in fact be trivial to solve using superficial features of the dataset: we give an example for an NP-hard problem where the data generated by any efficient sampler is trivially easy to classify.
Finally, we
\else
We
\fi
 empirically demonstrate the pitfalls of such approaches when applied to Conjunctive Query Containment, showing how biased data generation leads us to overestimate performance.

We discuss implications and limitations of our results.

\textbf{Can we teach current DNNs to solve or approximate NP-hard problems?~}
In practice, it is hard to see how, at least not for supervised approaches.
Our results imply a sort of ``Catch-22'' when training models to solve or approximate NP-hard problems.
On the one hand, labeling sufficient data to train increasingly large networks is infeasible.
Accurate labeling requires non-polynomial samplers such as deterministic solvers, and the problem space grows exponentially large. 
Moreover, experience shows that such models generalize poorly when we increase the problem size~\cite{DBLP:journals/corr/GravesWD14,prates2019learning}, meaning that even if we succeed, we would need to obtain new training data for larger problems.
On the other hand, tractable procedures generate an easier sub-problem that is not NP-hard. Thus, even if the model perfectly captures the sub-problem presented in the training set, there is no reason to believe it would be able to tackle the full, original NP-hard problem.

\textbf{Does this mean DNNs cannot learn to approximate NP-hard problems?~}
It does not.
We only discuss the difficulty of obtaining training and testing data, and do not say what DNN can or cannot learn.
If we somehow obtain an accurately labeled and sufficiently large training set and use the right optimization procedure, we might be able to teach a DNN model to solve such a problem.
Similarly, whether an approximation scheme (e.g., PTAS) exists for any particular problem, and whether that approximation is learnable, is beyond the scope of our work.

\textbf{Can semi-supervised learning help?~}
Unfortunately, in so far as these methods are efficient samplers, our results apply -- meaning they are similarly biased.
For example, popular semi-supervised learning approaches such as MixMatch~\cite{Berthelot_MixMatch_2019} and ADASYN~\cite{He_adasyn_2008} essentially perform data augmentation: they take samples from the training set and mutate them.

\textbf{Are the sub-problems trivial to solve?~}
\ifhardlanguage
They can be, as shown in Section \ref{sec:tricky_langs},  but not necessarily.
However, our experience with another NP-hard problem leads us to suspect many NP-hard problems do suffer from this issue to some extent. 
\else
Not necessarily.
However, our experience with a different NP-hard problem leads us to suspect that many data generators do generate samples that can be decided in polynomial time. 
\fi
We intend to explore this question in future work.

\textbf{Is it impossible to use ML to solve hard problems?~}
Not at all. 
First, not all hard problems are NP-hard, and even when they are, the application might not require solving the full NP-hard problem.
For many applications, solving an easier sub-problem may be sufficient.
For example, optimal elastic image matching is NP-complete~\cite{KEYSERS2003445}, yet ML techniques excel at computer vision tasks.
Second, while many ML approaches require a dense sampling from the modeled distribution, this does not necessarily apply to all approaches. 
A model that can learn from a very sparse sampling of the problem space could, presumably, be trained using a smaller training set generated by a deterministic solver.
However, we conjecture that such models must incorporate a non-polynomial deterministic solver of some sort. 
Rather than learning the problem directly, they could learn a polynomial reduction from the original problem to one the solver layer can solve. 
In particular, we believe hybrid architectures such as the differentiable SAT solver~\cite{pmlr-v97-wang19e} are a promising direction.

\textbf{Where to go from here?~}
As mentioned, we believe hybrid models that incorporate deterministic solvers or provable approximations might be one way forward.
Exchangable networks have also shown promising generalization to larger problem sizes~\ \cite{DBLP:conf/aaai/CameronCHL20}.
In addition, our theoretical results only apply to data generators that provide labels.
Methods that do not require labels, such as reinforcement learning, data do not suffer from the scaling problem~\cite{joshi2019learning,joshi2020learning}.
% the 2019 paper suggests 2 things:
% (1) reinforcement learning (RL) methods for regression problems don't require a labeled training set.
% (2) RL methods mught scale better to larger sizes than SL methods.
However, improvements should not be limited to model architectures and training paradigms, but look to sampling methods as well.
We intend to study the connection between the hardness of sampling and hardness of solving, and to quantify the hardness of the resulting sub-problem.% (e.g., how much of the full problem space it covers).

%% file: nonsample.tex
\subsection*{Proof of Lemma~\ref{lemma:non_sample}.}
\setcounter{theorem}{3}
\begin{lemma}
There exists an \NP-hard language ${L_1}$ and a function $\delta(n) \rightarrow 0$ as $n \rightarrow \infty$, such that for any sufficiently long $w$ generated by any randomized polynomial process, 
\[ \Pr[ w \in L_1 ] \leq \delta(n) \text{ .} \]
\end{lemma}

The proof is similar to the proof of Theorem 1 in \cite{itsykson_et_al:LIPIcs:2016:6808}.
The main difference is that we construct a decidable language, in contrast to the language generated in \cite{itsykson_et_al:LIPIcs:2016:6808}.

\begin{proof}
For every $n$, the output of a randomized algorithm $P$  is a random variable $P_n$: for $w \in \{0,1\}^n$, $\Pr[P_n = w]$ is the probability that given the length $n$, $P$ outputs $w$.
Let $K \subseteq \{0,1\}^n$ be a set of words of length $n$; $\Pr[P_n \in K]$ is the probability that a random word $w$ drawn by $P_n$ is in $K$.

Given two random variables $X,Y$ such that $X$, $Y$ take values in $\{0,1\}^n$, the \emph{statistical distance} between $X$ and $Y$ is defined as~\cite{itsykson_et_al:LIPIcs:2016:6808}: 
\begin{align*}
     \Delta(X, Y) = 
     \max_{K \subseteq \{0,1\}^n} \left\vert P[X \in K] - P[Y \in K] \right\vert \text{~.} 
\end{align*}

Using Theorem 9 in \cite{itsykson_et_al:LIPIcs:2016:6808} when $a = \frac{1}{2}$ and $b = 1$ we obtain the following corollary. 

\begin{corollary} \label{cor:sketch}
For every randomized algorithm $P$ that runs in time $O(n^{\log^{0.5} n})$ there exist infinitely many words that $P$ can only generate with probability less than $\epsilon(n)$, where $\epsilon(n) \rightarrow 0$ as $n \rightarrow \infty$.
\end{corollary}

We construct the randomized algorithm $P$ as follows.
Let $\mathcal{M}$ be an enumeration of all probabilistic Turing machines $\mathcal{M} =  M_1, M_2, M_3,...$, under a standard enumeration of Turing machines, and let $g(n)$ be a function that satisfies $g(n) \epsilon(n) \rightarrow 0$ and $g(n) \rightarrow \infty$ (where $\epsilon(n)$ is the function from Corollary~\ref{cor:sketch}). Example of such function is $g(n) = \frac{1}{\log(\epsilon(n))}$.
We define $\delta(n) = g(n) \epsilon(n)$, by the definition of $g(n)$, $\delta(n) \rightarrow 0$.

On input $n$, the algorithm $P$ uniformly chooses $M_i$ for $1 \leq i \leq g(n)$ and runs $M_i$ on the input $n$ (with the random bits $M_i$ needs) for $O(n^{\log^{0.5} n})$ steps. 
If $M_i$ returned a word $w < n$, $P$ pads it with $n - |w|$ zeros and returns the result. 
If $M_i$ returned a word $w > n$, $P$ trims $|w| - n$ characters from $w$ and returns it. 
Finally, if $M_i$ did not halt, $P$ returns $w = 1^n$. 

$P$ satisfies the following properties:
\begin{enumerate} 
    \item For every randomized polynomial algorithm $P'$ and for every $w \in \{0,1\}^n$ when $n$ is large enough, 
    \[ \Pr[P_n=w] \geq \frac{1}{g(n)} \Pr[P'_n=w] \text{ .} \] \label{property}
    \item $P$ runs in time $O(n^{\log^{0.5} n})$.
\end{enumerate}

We show that the first property holds as follows. 
Let $P'$ be a randomized polynomial algorithm that runs in time $O(n^c)$, and let $n_0$ be the first index that $P'$ appears in the enumeration $\mathcal{M}$.
For $w$, $|w| = n \geq g(n_0)$ and $n^{\log^{0.5} n} \geq n^c$, the probability of $P$ to generate $w$ is at least the probability to choose the machine $P'$, $\frac{1}{g(n)}$, multiplied by the probability that the machine $P'$ generates $w$: $\Pr[P'_n= w]$.
Note we give $P'$ enough time to complete the computation by choosing $n$ such that $n^{\log^{0.5} n} \geq n^c$.

The second property holds by the definition of $P$.

By Corollary \ref{cor:sketch} there exists a randomized algorithm $P^*$ such that for infinitely many $n$'s $n_1, n_2, n_3, ...$, it holds that $\Delta(P^*_n, P_n) \geq 1 - \epsilon(n)$. 
It means that for each such $n$, there exists a set of strings $K_n$ such that $\Pr[P_n \in K_n] \leq \epsilon(n)$.

Define $L_1$ as the union of all $K_n$.

Let $w \in L_1$ of length $n$ for sufficiently large $n$, and let $P'$ be a randomized polynomial algorithm.
\begin{align}
    \Pr[w = P'_n] &\leq g(n) \Pr[w = P_n] \\
                &\leq g(n)\epsilon(n) \\
                & = \delta(n) \rightarrow 0 \text{ .}
\end{align}
Where (1) follows from the first property of $P$, (2) follows from the definition of $L$, and (3) is the definition of $\delta(n)$. 

\end{proof}

%% file: cqc_details.tex
\subsection*{Additional Details on CQC}
For reproducibility, we include full details of our case study on  Conjunctive Query Containment (CQC).

\paragraph{Encoding Query Tokens}

Table~\ref{tab:model:representation} shows the mapping between query tokens and their representation as one-hot vectors.

\input{cqc_token_encoding}
\paragraph{Sampling Balanced Query Pairs from $\mu$}

We exploit the the phase transition phenomenon to define a parametric family of query pairs $\mu(m_1,m_2)$ such that sampling $(p,q)$ from $\mu(m_1,m_2)$ with $m_1 \geq m_2$ guarantees the following:
\begin{itemize}[nosep, noitemsep]
    \item $p$ has $m_1$ conjunctions and $q$ has $m_2$ conjunctions.
    \item The probability that $p \subset q$ is approximately 0.5.
    \item The process for generating positive and negative examples is the same.
\end{itemize}

\input{cqc_sampling_mu}
\paragraph{Data Augmentation for Conjunctive Query Pairs}

\input{cqc_data_augmentation}

%% file: ms.bbl
\begin{thebibliography}{53}
\providecommand{\natexlab}[1]{#1}
\providecommand{\url}[1]{\texttt{#1}}
\expandafter\ifx\csname urlstyle\endcsname\relax
  \providecommand{\doi}[1]{doi: #1}\else
  \providecommand{\doi}{doi: \begingroup \urlstyle{rm}\Url}\fi

\bibitem[Berthelot et~al.(2019)Berthelot, Carlini, Goodfellow, Papernot,
  Oliver, and Raffel]{Berthelot_MixMatch_2019}
Berthelot, D., Carlini, N., Goodfellow, I., Papernot, N., Oliver, A., and
  Raffel, C.~A.
\newblock {MixMatch}: A holistic approach to semi-supervised learning.
\newblock In \emph{Advances in Neural Information Processing Systems 32}, pp.\
  5050--5060. Curran Associates, Inc., 2019.

\bibitem[Budinich(1997)]{Budinich1997NeuralNF}
Budinich, M.
\newblock Neural networks for {NP}-complete problems.
\newblock \emph{Nonlinear Analysis: Theory, Methods \& Applications},
  30\penalty0 (3):\penalty0 1617--1624, 1997.

\bibitem[Cameron et~al.(2020)Cameron, Chen, Hartford, and
  Leyton{-}Brown]{DBLP:conf/aaai/CameronCHL20}
Cameron, C., Chen, R., Hartford, J.~S., and Leyton{-}Brown, K.
\newblock Predicting propositional satisfiability via end-to-end learning.
\newblock In \emph{The Thirty-Fourth {AAAI} Conference on Artificial
  Intelligence, {AAAI} 2020}, pp.\  3324--3331. {AAAI} Press, 2020.

\bibitem[Cawley \& Talbot(2010)Cawley and Talbot]{cawley2010over}
Cawley, G.~C. and Talbot, N.~L.
\newblock On over-fitting in model selection and subsequent selection bias in
  performance evaluation.
\newblock \emph{Journal of Machine Learning Research}, 11\penalty0
  (Jul):\penalty0 2079--2107, 2010.

\bibitem[Chandra \& Merlin(1977)Chandra and Merlin]{chandra1977optimal}
Chandra, A.~K. and Merlin, P.~M.
\newblock Optimal implementation of conjunctive queries in relational data
  bases.
\newblock In \emph{Proceedings of the ninth annual ACM symposium on Theory of
  computing}, pp.\  77--90. ACM, 1977.

\bibitem[Chiba \& Nishizeki(1985)Chiba and Nishizeki]{chiba1985arboricity}
Chiba, N. and Nishizeki, T.
\newblock Arboricity and subgraph listing algorithms.
\newblock \emph{SIAM Journal on computing}, 14\penalty0 (1):\penalty0 210--223,
  1985.

\bibitem[Chirkova(2018)]{Chirkova2018}
Chirkova, R.
\newblock \emph{Query Containment}, pp.\  2981--2985.
\newblock Springer New York, New York, NY, 2018.

\bibitem[Chollet(2017)]{10.5555/3203489}
Chollet, F.
\newblock \emph{Deep Learning with Python}.
\newblock Manning Publications Co., USA, 1st edition, 2017.
\newblock ISBN 1617294438.

\bibitem[Cook \& Mitchell(1997)Cook and Mitchell]{cook1997finding}
Cook, S.~A. and Mitchell, D.~G.
\newblock Finding hard instances of the satisfiability problem: A survey.
\newblock 1997.

\bibitem[Davis(2019)]{davis2019use}
Davis, E.
\newblock The use of deep learning for symbolic integration: A review of
  {(Lample and Charton, 2019)}.
\newblock \emph{arXiv preprint arXiv:1912.05752}, 2019.

\bibitem[Draper et~al.(2001)Draper, Halevy, and Weld]{draper2001nimble}
Draper, D., Halevy, A.~Y., and Weld, D.~S.
\newblock The nimble {XML} data integration system.
\newblock In \emph{Data Engineering, 2001. Proceedings. 17th International
  Conference on}, pp.\  155--160. IEEE, 2001.

\bibitem[Dressel \& Farid(2018)Dressel and Farid]{dressel2018accuracy}
Dressel, J. and Farid, H.
\newblock The accuracy, fairness, and limits of predicting recidivism.
\newblock \emph{Science Advances}, 4\penalty0 (1), 2018.

\bibitem[Escamocher et~al.(2019)Escamocher, O'Sullivan, and
  Prestwich]{DBLP:journals/corr/abs-1903-03592}
Escamocher, G., O'Sullivan, B., and Prestwich, S.~D.
\newblock Generating difficult {SAT} instances by preventing triangles.
\newblock \emph{CoRR}, abs/1903.03592, 2019.

\bibitem[Ferber et~al.(2019)Ferber, Wilder, Dilina, and
  Tambe]{ferber2019mipaal}
Ferber, A., Wilder, B., Dilina, B., and Tambe, M.
\newblock {MIPaaL}: Mixed integer program as a layer.
\newblock \emph{arXiv preprint arXiv:1907.05912}, 2019.

\bibitem[Florescu et~al.(1999)Florescu, Levy, and
  Suciu]{fernandez1999verifying}
Florescu, D., Levy, A., and Suciu, D.
\newblock Verifying integrity constraints on web sites.
\newblock In \emph{In IJCAI}, 1999.

\bibitem[Gent \& Walsh(1994)Gent and Walsh]{gent1994sat}
Gent, I.~P. and Walsh, T.
\newblock The {SAT} phase transition.
\newblock In \emph{Proceedings of the 11th European Conference on Artificial
  Intelligence}, pp.\  105--109, 1994.

\bibitem[Graves et~al.(2014)Graves, Wayne, and
  Danihelka]{DBLP:journals/corr/GravesWD14}
Graves, A., Wayne, G., and Danihelka, I.
\newblock Neural {Turing} machines.
\newblock \emph{CoRR}, abs/1410.5401, 2014.

\bibitem[Graves et~al.(2016)Graves, Wayne, Reynolds, Harley, Danihelka,
  Grabska-Barwi{\'n}ska, Colmenarejo, Grefenstette, Ramalho, Agapiou,
  et~al.]{graves2016hybrid}
Graves, A., Wayne, G., Reynolds, M., Harley, T., Danihelka, I.,
  Grabska-Barwi{\'n}ska, A., Colmenarejo, S.~G., Grefenstette, E., Ramalho, T.,
  Agapiou, J., et~al.
\newblock Hybrid computing using a neural network with dynamic external memory.
\newblock \emph{Nature}, 538\penalty0 (7626):\penalty0 471, 2016.

\bibitem[Haanp{\"a}{\"a} et~al.(2006)Haanp{\"a}{\"a}, J{\"a}rvisalo, Kaski, and
  Niemel{\"a}]{haanpaa2006hard}
Haanp{\"a}{\"a}, H., J{\"a}rvisalo, M., Kaski, P., and Niemel{\"a}, I.
\newblock Hard satisfiable clause sets for benchmarking equivalence reasoning
  techniques.
\newblock \emph{Journal on Satisfiability, Boolean Modeling and Computation},
  2\penalty0 (1-4):\penalty0 27--46, 2006.

\bibitem[He et~al.(2008)He, Bai, Garcia, and Li]{He_adasyn_2008}
He, H., Bai, Y., Garcia, E.~A., and Li, S.
\newblock {ADASYN}: Adaptive synthetic sampling approach for imbalanced
  learning.
\newblock In \emph{2008 IEEE International Joint Conference on Neural Networks
  (IEEE World Congress on Computational Intelligence)}, pp.\  1322--1328, June
  2008.

\bibitem[Horie \& Watanabe(1997)Horie and Watanabe]{horie1997hard}
Horie, S. and Watanabe, O.
\newblock Hard instance generation for {SAT}.
\newblock In \emph{International Symposium on Algorithms and Computation}, pp.\
   22--31. Springer, 1997.

\bibitem[Imieli{\'n}ski \& Lipski(1988)Imieli{\'n}ski and
  Lipski]{imielinski1988incomplete}
Imieli{\'n}ski, T. and Lipski, W.
\newblock Incomplete information in relational databases.
\newblock In \emph{Readings in Artificial Intelligence and Databases}, pp.\
  342--360. Elsevier, 1988.

\bibitem[Itsykson et~al.(2016)Itsykson, Knop, and
  Sokolov]{itsykson_et_al:LIPIcs:2016:6808}
Itsykson, D., Knop, A., and Sokolov, D.
\newblock Complexity of distributions and average-case hardness.
\newblock In Hong, S.-H. (ed.), \emph{27th International Symposium on
  Algorithms and Computation (ISAAC 2016)}, volume~64 of \emph{Leibniz
  International Proceedings in Informatics (LIPIcs)}, pp.\  38:1--38:12,
  Dagstuhl, Germany, 2016. Schloss Dagstuhl--Leibniz-Zentrum fuer Informatik.

\bibitem[Jarke \& Koch(1984)Jarke and Koch]{jarke1984query}
Jarke, M. and Koch, J.
\newblock Query optimization in database systems.
\newblock \emph{ACM Computing surveys (CsUR)}, 16\penalty0 (2):\penalty0
  111--152, 1984.

\bibitem[Joshi et~al.(2019)Joshi, Laurent, and Bresson]{joshi2019learning}
Joshi, C.~K., Laurent, T., and Bresson, X.
\newblock On learning paradigms for the travelling salesman problem.
\newblock \emph{arXiv preprint arXiv:1910.07210}, 2019.

\bibitem[Joshi et~al.(2020)Joshi, Cappart, Rousseau, Laurent, and
  Bresson]{joshi2020learning}
Joshi, C.~K., Cappart, Q., Rousseau, L.-M., Laurent, T., and Bresson, X.
\newblock Learning {TSP} requires rethinking generalization.
\newblock \emph{arXiv preprint arXiv:2006.07054}, 2020.

\bibitem[Kaiser \& Sutskever(2015)Kaiser and Sutskever]{kaiser2015neural}
Kaiser, {\L}. and Sutskever, I.
\newblock Neural {GPUs} learn algorithms.
\newblock \emph{arXiv preprint arXiv:1511.08228}, 2015.

\bibitem[Keysers \& Unger(2003)Keysers and Unger]{KEYSERS2003445}
Keysers, D. and Unger, W.
\newblock Elastic image matching is {NP}-complete.
\newblock \emph{Pattern Recognition Letters}, 24\penalty0 (1):\penalty0 445 --
  453, 2003.

\bibitem[Khalil et~al.(2017)Khalil, Dai, Zhang, Dilkina, and
  Song]{khalil2017learning}
Khalil, E., Dai, H., Zhang, Y., Dilkina, B., and Song, L.
\newblock Learning combinatorial optimization algorithms over graphs.
\newblock In \emph{Advances in Neural Information Processing Systems}, pp.\
  6348--6358, 2017.

\bibitem[Kingma \& Ba(2014)Kingma and Ba]{kingma2014adam}
Kingma, D.~P. and Ba, J.
\newblock Adam: A method for stochastic optimization.
\newblock \emph{arXiv preprint arXiv:1412.6980}, 2014.

\bibitem[Kool et~al.(2018)Kool, Van~Hoof, and Welling]{kool2018attention}
Kool, W., Van~Hoof, H., and Welling, M.
\newblock Attention, learn to solve routing problems!
\newblock \emph{arXiv preprint arXiv:1803.08475}, 2018.

\bibitem[Kov{\'a}cs \& Voronkov(2013)Kov{\'a}cs and Voronkov]{kovacs2013first}
Kov{\'a}cs, L. and Voronkov, A.
\newblock First-order theorem proving and {Vampire}.
\newblock In \emph{International Conference on Computer Aided Verification},
  pp.\  1--35. Springer, 2013.

\bibitem[Lample \& Charton(2020)Lample and Charton]{lample2019deep}
Lample, G. and Charton, F.
\newblock Deep learning for symbolic mathematics.
\newblock In \emph{International Conference on Learning Representations}, 2020.
\newblock To appear.

\bibitem[Lu et~al.(2017)Lu, Pu, Wang, Hu, and Wang]{lu2017expressive}
Lu, Z., Pu, H., Wang, F., Hu, Z., and Wang, L.
\newblock The expressive power of neural networks: A view from the width.
\newblock In \emph{Advances in neural information processing systems}, pp.\
  6231--6239, 2017.

\bibitem[Mehrabi et~al.(2019)Mehrabi, Morstatter, Saxena, Lerman, and
  Galstyan]{mehrabi2019survey}
Mehrabi, N., Morstatter, F., Saxena, N., Lerman, K., and Galstyan, A.
\newblock A survey on bias and fairness in machine learning.
\newblock \emph{arXiv preprint arXiv:1908.09635}, 2019.

\bibitem[Milan et~al.(2017)Milan, Rezatofighi, Garg, Dick, and
  Reid]{milan2017data}
Milan, A., Rezatofighi, S.~H., Garg, R., Dick, A., and Reid, I.
\newblock Data-driven approximations to {NP}-hard problems.
\newblock In \emph{Thirty-First AAAI Conference on Artificial Intelligence},
  2017.

\bibitem[Oakden-Rayner et~al.(2019)Oakden-Rayner, Dunnmon, Carneiro, and
  R{\'e}]{oakden2019hidden}
Oakden-Rayner, L., Dunnmon, J., Carneiro, G., and R{\'e}, C.
\newblock Hidden stratification causes clinically meaningful failures in
  machine learning for medical imaging.
\newblock \emph{arXiv preprint arXiv:1909.12475}, 2019.

\bibitem[P\'{e}rez et~al.(2019)P\'{e}rez, Marinkovi\'{c}, and
  Barcel\'{o}]{prez2019turing}
P\'{e}rez, J., Marinkovi\'{c}, J., and Barcel\'{o}, P.
\newblock On the {Turing} completeness of modern neural network architectures,
  2019.

\bibitem[Peterson \& Anderson(1988)Peterson and Anderson]{anderson1988neural}
Peterson, C. and Anderson, J.
\newblock Neural networks and {NP}-complete optimization problems; a
  performance study on the graph bisection problem.
\newblock \emph{Complex Syst.}, 2\penalty0 (1):\penalty0 59–89, February
  1988.

\bibitem[Prates et~al.(2019)Prates, Avelar, Lemos, Lamb, and
  Vardi]{prates2019learning}
Prates, M., Avelar, P.~H., Lemos, H., Lamb, L.~C., and Vardi, M.~Y.
\newblock Learning to solve {NP}-complete problems: A graph neural network for
  decision {TSP}.
\newblock In \emph{Proceedings of the AAAI Conference on Artificial
  Intelligence}, volume~33, pp.\  4731--4738, 2019.

\bibitem[Raghu et~al.(2017)Raghu, Poole, Kleinberg, Ganguli, and
  Dickstein]{raghu2017expressive}
Raghu, M., Poole, B., Kleinberg, J., Ganguli, S., and Dickstein, J.~S.
\newblock On the expressive power of deep neural networks.
\newblock In \emph{Proceedings of the 34th International Conference on Machine
  Learning-Volume 70}, pp.\  2847--2854. JMLR. org, 2017.

\bibitem[Sanchis(1990)]{sanchis1990complexity}
Sanchis, L.~A.
\newblock On the complexity of test case generation for {NP}-hard problems.
\newblock \emph{Information Processing Letters}, 36\penalty0 (3):\penalty0
  135--140, 1990.

\bibitem[Selman et~al.(1996)Selman, Mitchell, and
  Levesque]{selman1996generating}
Selman, B., Mitchell, D.~G., and Levesque, H.~J.
\newblock Generating hard satisfiability problems.
\newblock \emph{Artificial intelligence}, 81\penalty0 (1-2):\penalty0 17--29,
  1996.

\bibitem[Selsam \& Bj{\o}rner(2019)Selsam and Bj{\o}rner]{selsam2019guiding}
Selsam, D. and Bj{\o}rner, N.
\newblock Guiding high-performance {SAT} solvers with unsat-core predictions.
\newblock In \emph{International Conference on Theory and Applications of
  Satisfiability Testing}, pp.\  336--353. Springer, 2019.

\bibitem[Selsam et~al.(2018)Selsam, Lamm, B{\"u}nz, Liang, de~Moura, and
  Dill]{selsam2018learning}
Selsam, D., Lamm, M., B{\"u}nz, B., Liang, P., de~Moura, L., and Dill, D.~L.
\newblock Learning a {SAT} solver from single-bit supervision.
\newblock \emph{arXiv preprint arXiv:1802.03685}, 2018.

\bibitem[Siegelmann \& Sontag(1991)Siegelmann and Sontag]{siegelmann1991turing}
Siegelmann, H.~T. and Sontag, E.~D.
\newblock {Turing} computability with neural nets.
\newblock \emph{Applied Mathematics Letters}, 4\penalty0 (6):\penalty0 77--80,
  1991.

\bibitem[Smith(1999)]{smith1999neural}
Smith, K.~A.
\newblock Neural networks for combinatorial optimization: a review of more than
  a decade of research.
\newblock \emph{INFORMS Journal on Computing}, 11\penalty0 (1):\penalty0
  15--34, 1999.

\bibitem[Wang et~al.(2019)Wang, Donti, Wilder, and Kolter]{pmlr-v97-wang19e}
Wang, P.-W., Donti, P., Wilder, B., and Kolter, Z.
\newblock {SATN}et: Bridging deep learning and logical reasoning using a
  differentiable satisfiability solver.
\newblock In Chaudhuri, K. and Salakhutdinov, R. (eds.), \emph{Proceedings of
  the 36th International Conference on Machine Learning}, volume~97 of
  \emph{Proceedings of Machine Learning Research}, pp.\  6545--6554, Long
  Beach, California, USA, 09--15 Jun 2019. PMLR.

\bibitem[Weiss et~al.(2018)Weiss, Goldberg, and
  Yahav]{weiss-etal-2018-practical}
Weiss, G., Goldberg, Y., and Yahav, E.
\newblock On the practical computational power of finite precision {RNN}s for
  language recognition.
\newblock In \emph{Proceedings of the 56th Annual Meeting of the Association
  for Computational Linguistics (Volume 2: Short Papers)}, pp.\  740--745,
  Melbourne, Australia, July 2018. Association for Computational Linguistics.

\bibitem[Xu et~al.(2005)Xu, Boussemart, Hemery, and Lecoutre]{xu2005simple}
Xu, K., Boussemart, F., Hemery, F., and Lecoutre, C.
\newblock A simple model to generate hard satisfiable instances.
\newblock \emph{arXiv preprint cs/0509032}, 2005.

\bibitem[Xu et~al.(2020)Xu, Li, Zhang, Du, ichi Kawarabayashi, and
  Jegelka]{Xu2020What}
Xu, K., Li, J., Zhang, M., Du, S.~S., ichi Kawarabayashi, K., and Jegelka, S.
\newblock What can neural networks reason about?
\newblock In \emph{International Conference on Learning Representations}, 2020.

\bibitem[Yapo \& Weiss(2018)Yapo and Weiss]{yapo2018ethical}
Yapo, A. and Weiss, J.
\newblock Ethical implications of bias in machine learning.
\newblock 2018.

\bibitem[Zhang(2004)]{zhang2004phase}
Zhang, W.
\newblock Phase transitions and backbones of the asymmetric traveling salesman
  problem.
\newblock \emph{Journal of Artificial Intelligence Research}, 21:\penalty0
  471--497, 2004.

\end{thebibliography}
